\documentclass{article}

 \PassOptionsToPackage{numbers, compress}{natbib}

 \usepackage[preprint]{nips_2018}

\usepackage[T1]{fontenc}    
\usepackage{hyperref}       
\usepackage{url}            
\usepackage{booktabs}       
\usepackage{amsfonts}       
\usepackage{nicefrac}       
\usepackage{microtype}      

\usepackage{microtype}
\usepackage{graphicx}
\usepackage{subfigure}
\usepackage{booktabs} 
\usepackage{amsfonts}
\usepackage{amsmath} 

\usepackage{hyperref}

\usepackage{amsmath,amssymb,amsthm}
\usepackage{booktabs} 
\usepackage{amscd}
\usepackage{mathrsfs} 
\usepackage[shortlabels]{enumitem}
\usepackage{euscript}
\usepackage{graphicx}
\usepackage{amscd,bm}

\usepackage{calrsfs}

\DeclareMathAlphabet{\pazocal}{OMS}{zplm}{m}{n}
\DeclareMathAlphabet\boldsymbolcal{OMS}{cmsy}{b}{n}

\usepackage{algorithm}
\usepackage{algorithmic}
\newtheorem{theorem}{Theorem}
\newtheorem{definition}{Definition}
\newtheorem{lemma}{Lemma}
\newtheorem{proposition}{Proposition}
\newtheorem{remark}{Remark}

\providecommand{\nor}[1]{\left\lVert {#1} \right\rVert}

\providecommand{\scalT}[2]{\left\langle{#1},{#2}\right\rangle}

\def\bit{\begin{itemize}}
\def\eit{\end{itemize}}

\def\ben{\begin{enumerate}}
\def\een{\end{enumerate}}

\usepackage{listings}
\usepackage{color}

\definecolor{dkgreen}{rgb}{0,0.6,0}
\definecolor{gray}{rgb}{0.5,0.5,0.5}
\definecolor{mauve}{rgb}{0.58,0,0.82}

\usepackage{enumitem}

\usepackage{natbib}

\title{Regularized Finite Dimensional Kernel Sobolev Discrepancy}

\author{
  Youssef Mroueh\\
 IBM Research\\
IBM T.J Watson Research Center\\
  \texttt{mroueh@us.ibm.com} \\
}

\begin{document}

\maketitle

\begin{abstract}
We show in this note that the Sobolev Discrepancy introduced in \cite{SobolevGAN}  in the context of generative adversarial networks, is actually the weighted negative Sobolev norm $\nor{.}_{\dot{H}^{-1}(\nu_q)}$, that is known to linearize the Wasserstein $W_2$ distance and plays a fundamental role in the dynamic formulation of optimal transport of Benamou and Brenier. Given a Kernel with \emph{finite dimensional feature map} we show that the Sobolev discrepancy can be approximated from finite samples. Assuming this discrepancy is finite, the  error depends on  the approximation error in the function space induced by the finite dimensional feature space kernel and on a statistical error due to the finite sample approximation. 
\end{abstract}

\section{Sobolev Discrepancy and Weighted Negative Sobolev Norms   }\label{sec:SobW2}
In this Section we review the Sobolev Discrepancy introduced in \cite{SobolevGAN}. 
Let $\pazocal{X}$ be  a compact space in $\mathbb{R}^d$ with lipchitz boundary $\partial \pazocal{X}$. 
We start by defining the Sobolev Discrepancy:
\begin{definition}[Sobolev Discrepancy \cite{SobolevGAN}. ] Let $\nu_{p},\nu_{q}$ be two measures defined on $\pazocal{X}$. We define the Sobolev Discrepancy as follows:
\begin{equation}
D: \pazocal{S}(\nu_{p},\nu_{q})= \sup_{f}\Big\{\mathbb{E}_{x\sim \nu_{p}}f(x)-\mathbb{E}_{x\sim \nu_{q}}f(x): f \in W^{1,2}_0(\pazocal{X},\nu_{q}) ,
 \mathbb{E}_{x\sim \nu_{q}}\nor{\nabla_xf(x) }^2 \leq 1\Big \} 
\label{eq:Sobolev}
\end{equation}
\vskip -0.2in
where $W^{1,2}_0(\pazocal{X},\nu_{q})= \{ f: \pazocal{X}\to \mathbb{R} , f \text{ vanishes at the boundary of } \pazocal{X} \text{ and }\mathbb{E}_{x\sim \nu_q} \nor{\nabla_xf(x)}^2< \infty\}$.
\end{definition}

We note here  that this Sobolev discrepancy is actually known and already studied in optimal transport and relates to the Wasserstein 2 distance and its dynamical form given by Benamou and Brenier \cite{Dynamictransport}. It is indeed defined through the weighted  negative Sobolev Norm :
\begin{definition}[Weighted Negative Sobolev Norm \cite{Villani,peyre2016comparison}]
For $\mu$ a positive measure on $\pazocal{X}$, For a signed measure $\chi$ on $\pazocal{X}$, the weighted negative Sobolev Norm is defined as follows:
$$\nor{\chi}_{\dot{H}^{-1}(\mu)}= \sup_{f, \int_{\pazocal{X}} \nor{\nabla_x f(x)}^2 d\mu(x)\leq 1}\left| \int_{\pazocal{X}}f(x)d\chi(x)\right|.$$
$\nor{\chi}_{\dot{H}^{-1}(\mu)}$ is the dual norm of the  weighted Sobolev semi-norm $\nor{f}_{\dot{H}(\mu)}=\int_{\pazocal{X}}\nor{\nabla_x f(x)}^2d\mu(x)$. This norm is finite for measures of zero total mass, and can be infinite.
\end{definition}

It follows therefore that the Sobolev discrepancy corresponds to the $\dot{H}^{-1}(\nu_q)$ norm:
\begin{equation}
\pazocal{S}(\nu_{p},\nu_{q})=\nor{\nu_p-\nu_q}_{\dot{H}^{-1}(\nu_q)}.
\label{eq:weightedSobolevNorm}
\end{equation}

Note that $\dot{H}^{-1}(\nu_q)$ is different from the negative Sobolev norm $\dot{H}^{-1}(\pazocal{X})$ defined as follows  for a signed measure $\chi$ (See for instance  \cite{santambrogio2015optimal} chapter 5 Section 5.5.2)
$$\nor{\chi}_{\dot{H}^{-1}(\pazocal{X})}=\sup_{f, \int_{\pazocal{X}}\nor{\nabla_x f(x)}^2 dx \leq 1} \left| \int_{\pazocal{X}}f(x)d\chi(x)\right|$$
Note the negative Sobolev norm defines the following distance between $\nu_p$ and $\nu_q$:
\begin{equation}
\nor{\nu_p-\nu_q}_{\dot{H}^{-1}(\pazocal{X})} = \sup_{f, \int_{\pazocal{X}}\nor{\nabla_x f(x)}^2 dx \leq 1}\int_{\pazocal{X}}f(x)d(\nu_p(x)-\nu_q(x)).
\label{eq:Negative SobolevNorm}
\end{equation}
Note that  negative Sobolev norms   relate to Energy distances and MMD (See for instance \cite{peyre2017computational} page 119)
\section{Dual  and Primal Formulations of the Sobolev Discrepancy }
Following \cite{santambrogio2015optimal} (Chapter 5 Section 5.5.2) that characterizes the solution of problem \ref{eq:Negative SobolevNorm} via a diffusion PDE,  we  characterize in this  Section the solution of Problem \eqref{eq:Sobolev} via an advection PDE. 
This connection to  advection PDE was already given  in \cite{SobolevGAN} and is actually also a known result in the optimal transport literature. We simplify here the proofs and  notations, and give a primal $\inf$ formulation   for the  $\sup$ ``dual'' formulation in \eqref{eq:Sobolev}. 


\begin{proposition}[Primal and Dual Formulation of Sobolev Discrepancy]The following holds true:

1) \textbf{Dual  witness function from an advection PDE:} Let $u_{p,q}$ be the solution of the following advection PDE:
\begin{equation}
 p(x)-q(x)=- \text{div} (q(x)\nabla_x u(x)) ~~ u\big|_{\partial \pazocal{X}}=0 ,
 \label{eq:PDEsolution}
 \end{equation}
 physically this PDE means that we are moving the mass $q$ to $p$ following the flows of a velocity field given by $\nabla_x u$.
We have $\pazocal{S}^2(\nu_p,\nu_q)=\int_{\pazocal{X}} \nor{\nabla_x u_{p,q}}^2 q(x) dx,$ and the witness function of the Dual formulation of  \eqref{eq:Sobolev} is given by:
$f^*_{\nu_p,\nu_q}(x)=\frac{u_{p,q}(x)}{\pazocal{S}(\nu_p,\nu_q)}.$
 
2) \textbf{Sobolev Discrepancy as a minimum kinetic energy under an advection transport of $\nu_q$ to $\nu_p$:}   $\pazocal{S}(\nu_p,\nu_q)$  given in  \eqref{eq:Sobolev} admits the following equivalent primal formulation:
\begin{equation}
 P: \pazocal{S}(\nu_p,\nu_q)=\inf_{\bm{v}:\pazocal{X}\to \mathbb{R}^d }\Big \{ \sqrt{\int_{\pazocal{X}}\nor{\bm{v}(x)}^2 q(x) dx} \text{ subject to}: p(x)-q(x)=- \text{div} (q(x) \bm{v}(x))   \Big \} 
 \label{eq:PrimalSobolev}
\end{equation}
  that is the minimum kinetic energy to advect the mass $q$ to $p$ following the velocity field given by $\bm{v}$. The optimal velocity field is given by $\bm{v}^*=\nabla_x u_{p,q}$, where $u_{p,q}$ is the solution of the advection PDE \eqref{eq:PDEsolution}.

\label{prop:PrimDual}
\end{proposition}

\noindent It is important to note that the primal formulation \eqref{eq:PrimalSobolev} gives the transport interpretation of the Sobolev Discrepancy . $\pazocal{S}^2(\nu_p,\nu_q)$ is the minimum kinetic energy to transport the mass from the source measure $\nu_q$ to the target measure $\nu_p$ via an advection equation \eqref{eq:PDEsolution} with a velocity field $\bm{v}=\nabla_x u$. The kinetic energy is measured with the velocity field $\bm{v}$ and the mass  of the source distribution $q$. The optimal velocity field is a gradient $\bm{v}^*(x)=\nabla_x u_{p,q}(x)$, and $\pazocal{S}^2(\nu_p,\nu_q)=\int_{\pazocal{X}}\nor{\nabla_x u_{p,q}(x)}^2q(x)dx$.
The dual form  \eqref{eq:Sobolev} is computationally friendly since both its objective and constraints  can be expressed as expectations on $\nu_p$
 and $\nu_{q}$, and can be seen as a regularized mean discrepancy. Those computational properties were exploited in \cite{SobolevGAN}.

\begin{remark} As a summary of this section the Sobolev discrepancy admits three equivalent interpretations:
1) A ``Computational friendly'' regularized mean discrepancy  in its dual form given in   \eqref{eq:Sobolev}.
2) A minimum Kinetic Energy in a transport via advection in its primal  given in   \eqref{eq:PrimalSobolev}.
3) As a discrepancy computed with weighted Negative Sobolev Norms $\nor{.}_{\dot{H}^{-1}(\nu_q)}$ given in Equation \eqref{eq:weightedSobolevNorm}.
\end{remark}

\section{Sobolev Discrepancy and the Wasserstein 2 Distance }

\subsection{Wasserstein 2: Static and Dynamic Formulation}
Given the transport via advection interpretation of the Sobolev discrepancy, we think in the following  of $\nu_q$ as the source distribution and $\nu_p$ as the target distribution. The relation between weighted negative Sobolev norms and the Wasserstein 2 distance is well established in the optimal transport literature, since it linearizes the Wasserstein 2 distance \cite{Villani}. 
The Wasserstein 2 distance is defined as follows:
$$W_2(\nu_q,\nu_p)=\Big \{ \min_{(X,Y)}  \mathbb{E}_{(X,Y)}\nor{X-Y}^2_2, ~ X \sim \nu_q , Y \sim \nu_p \Big\}$$
For a small perturbation $\chi$ and any measure $\mu$ and a small $\varepsilon$ (See for instance \cite{Villani}):
$$W_2(\mu,\mu+\varepsilon \chi)= \varepsilon \nor{\chi}_{\dot{H}^{-1}(\mu)}+o(\varepsilon).$$
This identity is at the heart of the dynamic formulation of optimal transport \cite{Dynamictransport}:
$$W_2(\nu_q,\nu_p)=\int_{0}^{1} \nor{d\mu_{q_t}}_{\dot{H}^{-1}(\mu_{q_t})}, \mu_0=q, \mu_1=p.$$

The dynamic formulation as  given by Benamou and Brenier \cite{Dynamictransport}, finds a path of densities for transporting $q$ to $p$ via advection while  minimizing the kinetic energy  $\nor{d\mu_{q_t}}_{\dot{H}^{-1}(\mu_{q_t})}$. This can be written in the following equivalent form. For $t \in [0,1]$, let $v_t : \pazocal{X}\to \mathbb{R}^d$ be velocity fields  and $\mu_{q_t}$ be intermediate measures whose densities are $q_t$, we have:
 \begin{equation}
{W}^2_2(\nu_q,\nu_p)= \min_{q_t, \bm{v}_t}\Big\{ \int_{0}^1 \int_{\pazocal{X}}\nor{\bm{v}_t(x)}^2 q_t(x) dx dt, ~~ \frac{\partial q_t}{\partial t}= -div(q_t \bm{v}_t), q_{t=0}=q, q_{t=1}=p. \Big\}
\label{eq:W2dynamic}
\end{equation}
Note  that the expression given in Equation \eqref{eq:W2dynamic}  is exploiting  the primal kinetic energy formulation of the \emph{Sobolev discrepancy} given in Equation \eqref{eq:PrimalSobolev}. 
Peyre \cite{peyre2016comparison} exploited this connection between the Wasserstein distance and the the weighted negative Sobolev norm, to give upper and lower bounds on $W_2$ and $\nor{.}_{\dot{H}^{-1}(\nu_q)}$.
In the following, we give upper and lower bounds on $W_2$ and the Sobolev Discrepancy  $\pazocal{S}(\nu_p,\nu_q)$, while imposing stronger assumption on the boundedness of the density as done in \cite{santambrogio2015optimal} (Chapter 5, Section 5.5.2). Note that \cite{santambrogio2015optimal} gives upper and lower bounds for Negative Sobolev norms $\nor{.}_{\dot{H}^{-1}(\pazocal{X})}$ and not for the weighted case  $\nor{.}_{\dot{H}^{-1}(\nu_q)}$ as done in \cite{peyre2016comparison}.

\subsection{Bounding ${W}_2$ with Sobolev Discrepancy}
The following proposition shows that under some regularity conditions the Wasserstein 2 distance can be upper and lower bounded by the Sobolev Discrepancy. 
\begin{proposition} Assume that $\nu_p,\nu_q$ are absolutely continuous measures , with densities bounded from above and below by two constants $(0<a<b<m)$. Then we have:
$$\sqrt{\frac{a}{b}} \pazocal{S}(\nu_p,\nu_q) \leq W_2(\nu_q,\nu_p)\leq 2  \pazocal{S}(\nu_p,\nu_q).$$
\label{prop:BoundW2}
\end{proposition}
\vskip -0.2in
From Proposition \ref{prop:BoundW2} we see that  the Wasserstein 2 distance $W_2$ and the Sobolev Discrepancy are equivalent under some regularity assumptions on the density.  

\subsection{Unconstrained Form of $\pazocal{S}^2(\nu_p,\nu_q)$}
We end this Section with an unconstrained equivalent form for the Sobolev discrepancy that will prove to be useful for the our future developments in the paper:

\begin{lemma} The following equivalent form holds true for the squared Sobolev Discrepancy:
\begin{equation}
\pazocal{S}^2(\nu_p,\nu_q)=\sup_{u \in W^{1,2}_{0}(\pazocal{X},\nu_q)} \left\{L(u)=2\int_{\pazocal{X}} u(x) d(\nu_p(x)-\nu_q(x))-\int_{\pazocal{X}}\nor{\nabla_x u(x)}^2 d\nu_q(x)\right\},
\label{eq:variationalEqW2}
\end{equation}
the optimal $u^*$ is given by $u_{p,q}$ solution of the advection PDE \eqref{eq:PDEsolution}. 
Moreover we have for any feasible $u$:
\begin{equation}
\pazocal{S}^2(\nu_p,\nu_q)-L(u)=\int_{\pazocal{X}}\nor{\nabla_x u(x)-\nabla_x u_{p,q}(x)}^2q(x)dx.
\label{eq:comparaisonequality}
\end{equation}
\label{lem:Ssquared}
\end{lemma}
Given in this form we see that the main computational difficulty in computing the Sobolev Discrepancy is in optimization over the space $W^{1,2}_{0}(\pazocal{X},\nu_q)$. \cite{SobolevGAN} proposed to parametrize  functions with neural networks. In this paper we propose to relax this function space to a Reproducing Kernel Hilbert Space (RKHS) $\mathcal{H}$, with the goal of having certain of the nice theoretical propreties of the Sobolev Discrepancies carrying on to the Kernelized case.

\section{Kernelized Sobolev Discrepancy }\label{sec:KSD}
In this Section we define the Kernelized Sobolev Discrepancy  by looking for the optimal witness funtion of \eqref{eq:variationalEqW2} in a Hypothesis function class that is a \textbf{Finite dimensional} Reproducing Kernel Hilbert Space (RKHS). We start first by reviewing some RKHS properties and assumptions needed for our statements. 

Let $\mathcal{H}$ be a Reproducing Kernel Hilbert Space with an associated finite feature map $\Phi:\pazocal{X}\to \mathbb{R}^m$. The associated kernel $k$ is therefore $k(x,y)=\scalT{\Phi(x)}{\Phi(y)}_{\mathcal{H}}=\sum_{j=1}^m \Phi_{j}(x)\Phi_{j}(y)$. Note that for any function $f\in \mathcal{H}$, we have $f(x)=\scalT{\boldsymbol{f}}{\Phi(x)}$, where $\boldsymbol{f}\in \mathbb{R}^m$,  and $\scalT{}{}$ is the dot product in $\mathbb{R}^m$. We note $\nor{f}_{\mathcal{H}}=\nor{\boldsymbol{f}}=\sum_{j=1}^m \boldsymbol{f}^2_j$.  Let $J\Phi(x) \in \mathbb{R}^{d\times m}$ be the jacobian of $\Phi$, $[J\Phi]_{a,j}(x)=\frac{\partial}{\partial x_a}\Phi_j(x) $.
We have the following expression of the gradient  $\nabla_x f(x)=(J\Phi(x)\boldsymbol{f}) \in \mathbb{R}^m$. 

We make the following assumptions on $\mathcal{H}$:
\begin{enumerate}
\item[A1] There exists $\kappa_1<m$ such that $\sup_{x\in \pazocal{X}} \nor{\Phi(x)}<\kappa_1$.
\item [A2] There exists $\kappa_2<m$ such that for all $a=1\dots d$:\\
$\sup_{x\in \pazocal{X}} Tr(\frac{\partial}{\partial x_a}\Phi(x)\otimes \frac{\partial}{\partial x_a}\Phi(x) )<\kappa_2$.
\item [A3] $\mathcal{H}$ vanishes on the boundary: $\forall j =1\dots m, \Phi_j(x)|_{\partial \pazocal{X}}=0$ (for $\pazocal{X}=\mathbb{R}^d$, it is enough to have $\lim_{x\to \infty} \Phi_{j}(x)=0$) .
\end{enumerate}

\subsection{Kernel Sobolev Discrepancy}
We define in what follows the Kernelized Sobolev Discrepancy by restricting the problem given in Equation to \eqref{eq:Sobolev} to functions in  a RKHS.

\begin{definition}[Kernelized Sobolev Discrepancy] Let $\mathcal{H}$ be a finite dimensional RKHS satisfying assumptions A1,A2 and A3. Let $\nu_{p},\nu_{q}$ be two measures defined on $\pazocal{X}$. We define the Sobolev discrepancy restricted to the space $\mathcal{H}$ as follows:
\begin{equation}
\pazocal{S}_{\mathcal{H}}(\nu_p,\nu_q)= \sup_{f}\Big \{ \mathbb{E}_{x\sim \nu_{p}}f(x)-\mathbb{E}_{x\sim \nu_{q}}f(x),  f \in\mathcal{H}, 
  \mathbb{E}_{x\sim \nu_{q}}\nor{\nabla_xf(x) }^2 \leq 1\Big \}
\label{eq:SobolevH}
\end{equation}
we note ${f}^{\mathcal{H}}_{\nu_p,\nu_q} \in \mathcal{H}$, the optimal witness function.
\end{definition}
Note $\Omega(f)=\underset{x\sim \nu_{q}}{\mathbb{E}}\nor{\nabla_xf(x) }^2$.  For $f \in \mathcal{H}$, we have: 

$$\Omega(f)=\sum_{a=1}^d \underset{x\sim \nu_{q}}{\mathbb{E}}\scalT{\boldsymbol{f}}{\frac{\partial \Phi(x)}{\partial x_a}}^2=\sum_{a=1}^d\underset{x\sim \nu_{q}}{\mathbb{E}}\scalT{\boldsymbol{f}}{\left(\frac{\partial \Phi(x)}{\partial x_a}\otimes \frac{\partial \Phi(x)}{\partial x_a}
\right)\boldsymbol{f}}=\scalT{f}{D(\nu_q))f}_{\mathcal{H}},$$ where  we identified an operator $D(\nu_q)$:
\vskip -0.2in
\begin{equation}
 D(\nu_q)=\mathbb{E}_{x\sim \nu_{q}} \sum_{a=1}^d\frac{\partial \Phi(x)}{\partial x_a}\otimes \frac{\partial \Phi(x)}{\partial x_a}= \mathbb{E}_{x\sim \nu_{q}} [J\Phi(x)]^{\top}J\Phi(x).
 \label{eq:KDGE}
\end{equation}
We call $D(\nu_q)$ the Kernel  Derivative Gramian Embedding \textbf{KDGE} of a distribution $\nu_q$ . KDGE is an operator  embedding of the distribution i, that takes the fingerprint of the distribution with respect to the feature map derivatives averaged over all coordinates. This operator embedding  of $\nu_q$ is to be contrasted with the classic Kernel Mean Embedding \textbf{KME} of a distribution in $\mathcal{H}$: $\mu(\nu_q)=\mathbb{E}_{x\sim \nu_q}\Phi(x)$.

\begin{lemma}[Unconstrained Form of Kernel Sobolev Discrepancy]
\begin{eqnarray}
\pazocal{S}^2_{\mathcal{H}}(\nu_p,\nu_q)&=&\sup_{u \in \mathcal{H}} \left\{2\int_{\pazocal{X}} u(x) d(\nu_p(x)-\nu_q(x))-\int_{\pazocal{X}}\nor{\nabla_x u(x)}^2 d\nu_q(x)\right\}\nonumber \\
&=&\sup_{\boldsymbol{u} \in \mathbb{R}^m} 2\scalT{\boldsymbol{u}}{\mu(\nu_p)-\mu(\nu_q)}-\scalT{\boldsymbol{u}}{(D(\nu_q)) \boldsymbol{u}},
\label{eq:variationalEqK}
\end{eqnarray}
where $\mu(\nu_p),D(\nu_q)$ are  the KME the KDGE defined above. Let $\boldsymbol{u}^*=\boldsymbol{u}^{\mathcal{H}}_{p,q}$, be the optimum.
\label{lem:unconstKernel}
\end{lemma}
\begin{proof} The proof follows from Proposition \ref{pro:RegSobWit},setting $\lambda=0$.
\end{proof}

\noindent Proposition \eqref{prop:SolKernel} gives the expression of the optimal  Kernel Sobolev witness function ${f}^{\mathcal{H}}_{\nu_p,\nu_q} \in \mathcal{H}$.
\begin{proposition} Assume that the KDGE of $\nu_q$, $D(\nu_q)$ defined in Equation \eqref{eq:KDGE} is non singular. Let $\mu(\nu_p)$ and $\mu(\nu_q)$ be the KME of $\nu_p$ and $\nu_q$ respectively.
The solution of Problem \eqref{eq:variationalEqK}, $\boldsymbol{u}^{\mathcal{H}}_{\nu_p,\nu_q}$ in $\mathbb{R}^m$  is given by:
\begin{equation}
 \boldsymbol{u}^{\mathcal{H}}_{p,q}=\left[D(\nu_{q})\right]^{-1} \left(\mu(\nu_p)-\mu(\nu_q)\right),
 \label{eq:criticRKHS}
 \end{equation}
and $u^{\mathcal{H}}_{p,q}=\scalT{\boldsymbol{u}^{\mathcal{H}}_{p,q}}{\Phi(x)}$. 
$\pazocal{S}^2_{\mathcal{H}}(\nu_p,\nu_q)=\nor{\left[D(\nu_{q})\right]^{-\frac{1}{2}} \left(\mu(\nu_p)-\mu(\nu_q)\right) }^2=\int_{\pazocal{X}}\nor{\nabla_xu^{\mathcal{H}}_{p,q}(x)}^2q(x)dx,$
and the witness function  of Kernel Sobolev Discrepancy \eqref{eq:SobolevH} $\boldsymbol{f}^{\mathcal{H}}_{\nu_p,\nu_q} =\frac{\boldsymbol{u}^{\mathcal{H}}_{\nu_p,\nu_q}}{\pazocal{S}_{\mathcal{H}}(\nu_p,\nu_q)  }$
\label{prop:SolKernel}
\end{proposition}
Note that $u^{\mathcal{H}}_{p,q}$ is the approximation in $\mathcal{H}$ of $u_{p,q}$ the solution of the PDE \eqref{eq:PDEsolution}. We have from Lemma \ref{lem:Ssquared} (Equation \eqref{eq:comparaisonequality}):
$$ \pazocal{S}^2(\nu_p,\nu_q)- \pazocal{S}^2_{\mathcal{H}}(\nu_p,\nu_q)=\int_{\pazocal{X}}\nor{\nabla_x u^{\mathcal{H}}_{p,q}(x)-\nabla_x u_{p,q}(x)}^2q(x)dx,
$$
Hence the approximation in the space $\mathcal{H}$ is in the Sobolev semi-norm sense. We know that $\nabla_x u_{p,q}$ has  the physical interpretation of a velocity advecting the mass from $q$ to $p$. In the next Section we will take a close look at $\nabla_x u^{\mathcal{H}}_{p,q}$.
%
\subsection{Transport in RKHS: Understanding $D(\nu_q)$} \label{sec:transportRKHS}
Let $(\lambda_j,\boldsymbol{\psi}_j)$ be eigenvectors of $D(\nu_q)$ . Assume that $\lambda_j>0$ for all $j=1\dots m$. We have :
$D(\nu_q){\boldsymbol{\psi}}_j= \lambda_j \boldsymbol{\psi}_j,$
this means that $\scalT{\frac{1}{\sqrt{\lambda_k}}\boldsymbol{\psi}_k}{D(\nu_q)\frac{1}{\sqrt{\lambda_j}}\boldsymbol{\psi}_j}=\delta_{jk}$. Note $\tilde{\boldsymbol{\psi}}_j=\frac{1}{\sqrt{\lambda_j}}\boldsymbol{\psi}_j $. It is easy to see that this means:
$\scalT{\frac{1}{\sqrt{\lambda_k}}\boldsymbol{\psi}_k}{D(\nu_q)\frac{1}{\sqrt{\lambda_j}}\boldsymbol{\psi}_j}=\int_{\pazocal{X}} \scalT{\nabla_x \tilde{{\psi}}_j(x)}{\nabla_x \tilde{{\psi}}_k(x)}_{\mathbb{R}^d} q(x)dx=\delta_{jk},$
hence we have $\{\nabla_x \tilde{\boldsymbol{\psi}}_j\}^{m}_{j=1}$ are orthonormal in  $\mathcal{L}_2(\pazocal{X},\nu_q)^{\otimes d}$.
We have for all $a=1\dots d$:
\begin{align*}
&\frac{\partial u^{\mathcal{H}}_{p,q}(x)}{\partial x_a}
= \scalT{\boldsymbol{u}^{\mathcal{H}}_{\nu_p,\nu_q}}{\frac{\partial \Phi(x)}{\partial x_a}}
= \scalT{\left[D(\nu_{q})\right]^{-1} \left(\mu(\nu_p)-\mu(\nu_q)\right)}{\frac{\partial \Phi(x)}{\partial x_a}}\\
&= \scalT{\sum_{j=1}^{m}\frac{1}{\lambda_j}\boldsymbol{\psi}_j \boldsymbol{\psi}^*_j(\mu(\nu_p)-\mu(\nu_q))}{\frac{\partial \Phi(x)}{\partial x_a}}
= \sum_{j=1}^{m}\frac{1}{\lambda_j}\scalT{\boldsymbol{\psi}_j}{\mu(\nu_p)-\mu(\nu_q)}\scalT{\boldsymbol{\psi}_j}{\frac{\partial \Phi(x)}{\partial x_a}}\\
&= \sum_{j=1}^{m}\frac{1}{\lambda_j}\scalT{\boldsymbol{\psi}_j}{\mu(\nu_p)-\mu(\nu_q)} \frac{\partial \psi_j(x)}{\partial x_a}.
\end{align*}
Hence we have:
$\nabla_x u^{\mathcal{H}}_{p,q}(x)= \sum_{j=1}^{m} \scalT{\tilde{\boldsymbol{\psi}_j}}{\mu(\nu_p)-\mu(\nu_q)} \nabla_x \tilde{{\psi}_j}(x), $
hence $D(\nu_q)$ eigenvectors give raise to  a basis of principal transport directions $ \nabla_x \tilde{{\psi}_j}(x)= J\Phi(x) \boldsymbol{\tilde{\psi}_j}$.  A principal transport direction is weighted positively if $\scalT{\tilde{\boldsymbol{\psi}_j}}{\mu(\nu_p)-\mu(\nu_q)}>0$, meaning it is aligned with the KME differences in the direction of the desired transport from $q$ to $p$.

\subsection{Kernel Sobolev Discrepancy and the Wasserstein 2 Distance}
In this Section we show that if $W_2(\nu_q,\nu_p)=0$ this implies that the finite dimensional Kernel Sobolev Discrepancy is zero. Which means that a sequence is convergent in the Kernel Sobolev Discrepancy whenever it converges in the $W_2$ sense. 
\begin{proposition}[Convergence and Density] Assume $\nu_p$ and $\nu_q$ are continuous and bounded from above an below by $0<a<b$
 For a RKHS $\mathcal{H}$  with finite dimensional feature map satisfying Assumptions A1, A2 and A3. We have:
$\pazocal{S}_{\mathcal{H}}(\nu_p,\nu_q)\leq \pazocal{S}(\nu_p,\nu_q),$
and :
$$\sqrt{\frac{a}{b}}\pazocal{S}_{\mathcal{H}}(\nu_p,\nu_q) \leq {W}_2(\nu_q,\nu_p).$$
which means that a sequence $\nu_{q_{n}}$ (continuous with densities, bounded from above and below) is convergent in $\pazocal{S}_{\mathcal{H}}$, whenever it converges in the Wasserstein 2 ${W}_{2}$.

\label{prop:dense}
\end{proposition}

\section{Regularized Kernel Sobolev Discrepancy}
Regularization in the RKHS consists as we will see in avoiding singularity issues of the KDGE $D(\nu_q)$ and plays a fundamental role in stabilizing the computations of the Discrepancy. We define below the Regularized Kernel Sobolev Discrepancy (RKSD):
\begin{definition}[Regularized Kernel Sobolev Discrepancy (RKSD)] The RKSD is defined as follows,
\begin{equation}
\pazocal{S}_{\mathcal{H},\lambda}(\nu_p,\nu_q)=\sup_{f } \Big\{\mathbb{E}_{x\sim \nu_p}f(x)-\mathbb{E}_{x\sim \nu_q}f(x), f\in \mathcal{H}, \mathbb{E}_{x\sim \nu_q}\nor{\nabla_x f(x)}^2+\lambda \nor{f}^2_{\mathcal{H}} \leq 1\Big\},
\label{eq:ipmSobO}
\end{equation}
where $\lambda>0$ is the regularization parameter and $\nor{.}_{\mathcal{H}}$ is the RKHS norm. Let $f^{\lambda}_{\nu_p,\nu_q}$ be the optimal witness function.
\end{definition}

The following proposition summarizes the main properties of RKSD:
\begin{theorem}[Unconstrained Form of the RKSD/ witness function]
The squared RKSD has the following equivalent form.
\begin{eqnarray}
&\pazocal{S}^2_{\mathcal{H},\lambda}(\nu_p,\nu_q)=\sup_{u \in \mathcal{H}} \left\{2\int_{\pazocal{X}} u(x) d(\nu_p(x)-\nu_q(x))-\int_{\pazocal{X}}\nor{\nabla_x u(x)}^2 d\nu_q(x)-\lambda \nor{u}^2_{\mathcal{H}}\right\}\nonumber \\
&=\sup_{u \in \mathbb{R}^m}L(u,\lambda)= 2\scalT{\boldsymbol{u}}{\mu(\nu_p)-\mu(\nu_q)}-\scalT{\boldsymbol{u}}{(D(\nu_q)+\lambda I) \boldsymbol{u}}
\label{eq:variationalEqReg}
\end{eqnarray}
The following properties characterize the RKSD and its witness function

1) The optimal $u^*$ in \eqref{eq:variationalEqReg} is  $\boldsymbol{u}^{\lambda}_{p,q}=(D(\nu_q)+\lambda I)^{-1}(\mu(\nu_p)-\mu(\nu_q)$.\\
2) $\pazocal{S}^2_{\mathcal{H},\lambda}(\nu_p,\nu_q)=\nor{(D(\nu_q)+\lambda I)^{-\frac{1}{2}}(\mu(\nu_p)-\mu(\nu_q))}^2$.\\
3) $\pazocal{S}^2_{\mathcal{H},\lambda}(\nu_p,\nu_q)=\int_{\pazocal{X}} \nor{\nabla_x u^{\lambda}_{p,q}(x)}^2q(x)dx+ \lambda \nor{\boldsymbol{u}^{\lambda}_{p,q}}^2$, hence $\pazocal{S}^2_{\mathcal{H},\lambda}(\nu_p,\nu_q)$ is a regularized kinetic energy.\\
4)For any $u \in \mathcal{H}$ we have: $$\pazocal{S}^2_{\mathcal{H},\lambda}(\nu_p,\nu_q)-L(u,\lambda) =\nor{\sqrt{D(\nu(q))}(\boldsymbol{u}-\boldsymbol{u}^{\lambda}_{p,q})}^2+\lambda \nor{\boldsymbol{u} -\boldsymbol{u}^{\lambda}_{p,q}}^2 $$
5) $f^{\lambda}_{\nu_p,\nu_q}=\frac{\boldsymbol{u}^{\lambda}_{p,q}}{\pazocal{S}_{\mathcal{H},\lambda}(\nu_p,\nu_q)} $ is the optimal witness function of \eqref{eq:ipmSobO}.
\label{pro:RegSobWit}
\end{theorem}
We see in that case that the optimal witness function of $\pazocal{S}^2_{\mathcal{H},\lambda}(\nu_p,\nu_q)$ satisfies the following identity:
$$\boldsymbol{u}^{\lambda}_{p,q}=(D(\nu_q)+\lambda I)^{-1}(\mu(\nu_p)-\mu(\nu_q)),$$
and hence regularization amounts to regularizing the KDGE. Moreover $\pazocal{S}^2_{\mathcal{H},\lambda}(\nu_p,\nu_q)$ has the interpretation of a regularized kinetic energy.
We shall study the propreties of $\nabla_x {u}^{\lambda}_{p,q}$ as a transport map in the following Section.
\subsection{Regularized Transport in RKHS: Impact of regularization on the principal Transport directions}
Similarly to the un-regularized case consider $(\lambda_j,\boldsymbol{\psi}_j)$ eigenfunctions in $\mathcal{H}$ of $D(\nu_q)$ we have in this case:
$$\nabla_x {u}^{\lambda}_{p,q}(x)= \sum_{j=1}^{m} \frac{1}{\lambda_j+\lambda}\scalT{{\boldsymbol{\psi}_j}}{\mu(\nu_p)-\mu(\nu_q)} \nabla_x {{\psi}_j}(x),$$
hence we see that regularization is spectral filtering the principal transport directions $\nabla_x {\boldsymbol{\psi}_j}(x)$ weighing down small eigenvalues. Hence the impact of regularization here is similar to spectral filtering principal directions of the covariance matrix in kernel PCA, but here it is filtering principal transport directions $\nabla_x {\psi}_j$.

\section{Empirical Regularized Kernel Sobolev Discrepancy and Generalization Bounds }
We define below the Empirical Regularized Kernel Sobolev Discrepancy $\hat{\pazocal{S}}_{\mathcal{H},\lambda}(\hat{\nu}_p,\hat{\nu}_q)$ for empirical measures $\hat{\nu}_p,\hat{\nu}_q$.  We then give generalization bounds, i.e finite sample bounds on its convergence convergence to the Expected   Kernelized Sobolev Discrepancy ${\pazocal{S}}_{\mathcal{H}}({\nu}_p,{\nu}_q)$ and the Sobolev Discrepancy $\pazocal{S}(\nu_p,\nu_q)$. We then give a closed form solution of the empirical critic of the Kernelized Sobolev Discrepancy.
\begin{definition} [Regularized Empirical Kernelized Sobolev Discrepancy] Let $\{x_i,i=1\dots N, x_i \sim \nu_p\}$, and $\{y_j,j=1\dots M, y_j \sim \nu_q\}$, be samples from $\nu_p$
and $\nu_q$ respectively. Let $\hat{\nu}_p(x)=\frac{1}{N}\sum_{i=1}^N \delta(x-x_i)$ and $\hat{\nu}_q(y)=\frac{1}{M}\sum_{i=1}^M \delta(y-y_i)$. We define the regularized  empirical Kernelized Sobolev Discrepency as follows:  
\begin{equation}
\hat{\pazocal{S}}_{\mathcal{H},\lambda}(\hat{\nu}_p,\hat{\nu}_q)= \sup_{f \in {\mathcal{H}}} \Big\{\frac{1}{N} \sum_{i=1}^N f(x_i)-\frac{1}{M}\sum_{i=1}^M f(y_j): \frac{1}{M} \sum_{j=1}^M \nor{\nabla_x f(y_j)}^2+\lambda \nor{f}^2_{\mathcal{H}} \leq 1\Big \}
\label{eq:SobolevHEmp}
\end{equation}
achieved at $f^{\lambda}_{\hat{\nu_p},\hat{\nu_q}}$.
\end{definition}
Similarly to the expected  case the following lemma characterizes the witness function of $\hat{\pazocal{S}}^2_{\mathcal{H},\lambda}(\hat{\nu}_p,\hat{\nu}_q)$:
\begin{lemma}
1)We have  the following unconstrained equivalent form:
\begin{equation}
\hat{\pazocal{S}}^2_{\mathcal{H},\lambda}(\hat{\nu}_p,\hat{\nu}_q)= \sup_{u \in {\mathcal{H}}} \Big\{\frac{2}{N} \sum_{i=1}^N u(x_i)-\frac{2}{M}\sum_{i=1}^M u(y_j)-\frac{1}{M} \sum_{j=1}^M \nor{\nabla_x u(y_j)}^2-\lambda \nor{u}^2_{\mathcal{H}} \leq 1\Big \}
\label{eq:SobolevHEmp2}
\end{equation}
2) The optimal witness function of $\hat{\pazocal{S}}^{2}_{\mathcal{H},\lambda}(\hat{\nu}_p,\hat{\nu}_q)$ satisfies: $\boldsymbol{\hat{u}}^{\lambda}_{p,q}=(\hat{D}(\nu_q)+\lambda I)^{-1}(\hat{\mu}(\nu_p)-\hat{\mu}(\nu_q),$ where $\hat{D}(\nu_q)=\frac{1}{M} \sum_{j=1}^M  [J\Phi(y_j)]^{\top}J\Phi(y_j)$ is the empirical KDGE and $\hat{\mu}(\nu_p)=\frac{1}{N}\sum_{i=1}^N \Phi(x_i)$ and $\hat{\mu}(\nu_q)=\frac{1}{M}\sum_{j=1}^M \Phi(y_j)$ the empirical KMEs.\\
3) $\hat{\pazocal{S}}^2_{\mathcal{H},\lambda}(\hat{\nu}_p,\hat{\nu}_q)=\nor{(\hat{D}(\nu_q)+\lambda I)^{-\frac{1}{2}}(\hat{\mu}(\nu_p)-\hat{\mu}(\nu_q))}^2.$
\end{lemma}
\begin{proof} Apply proposition \ref{pro:RegSobWit}  for $\hat{\nu}_p(x)=\frac{1}{N}\sum_{i=1}^N \delta(x-x_i)$ and $\hat{\nu}_q(y)=\frac{1}{M}\sum_{i=1}^M \delta(y-y_i)$
\end{proof}

\subsection{Convergence analysis }
In this Section we want first a comparison inequality  between the squared  Kernel Sobolev Discrepancy  $ \pazocal{S}^2_{\mathcal{H}}(\nu_p,\nu_q)$  and the squared empirical Regularized Sobolev Discrepancy $\hat{\pazocal{S}}^2_{\mathcal{H},\lambda}(\nu_p,\nu_q)$ , through their respective witness functions $u^{\mathcal{H}}_{p,q}$ and $\hat{u}^{\lambda}_{p,q}$. The following Lemma establishes this relation:

\begin{lemma}[comparison inequalities] Assume $\pazocal{S}^2(\nu_p,\nu_q)<\infty$.
1) Approximation  error of $W^{1,2}_{0}(\pazocal{X},\nu_q)$ in $\mathcal{H}$: 
$$|\pazocal{S}^2(\nu_p,\nu_q)- \pazocal{S}^2_{\mathcal{H}}(\nu_p,\nu_q)|\leq \inf_{u\in \mathcal{H}} \int_{\pazocal{X}}\nor{\nabla_x u(x)-\nabla_x u_{p,q}(x)}^2q(x)dx
$$
2) Statistical Error, approximation with samples.
Note $\delta=\mu(\nu_p)-\mu(\nu_q)$ and $\hat{\delta}=\hat{\mu}(\nu_p)-\hat{\mu}(\nu_q)$.
 We have:
\begin{align*}
| \hat{\pazocal{S}}^2_{\mathcal{H},\lambda}(\hat{\nu}_p,\hat{\nu}_q)-\pazocal{S}^2_{\mathcal{H}}(\nu_p,\nu_q)| &\leq \nor{\delta-\hat{\delta}}\nor{\boldsymbol{\hat{u}}^{\lambda}_{p,q}}+(1+\lambda )\nor{\boldsymbol{\hat{u}}^{\lambda}_{p,q}}^2\nor{D(\nu_q)-\hat{D}(\nu_q)}_{op}\\
&+  \nor{\sqrt{D(\nu(q))}(\boldsymbol{\hat{u}}^{\lambda}_{p,q}-\boldsymbol{u}^{\mathcal{H}}_{p,q})}^2
\end{align*}
\label{lem:compa}
\end{lemma}

Assume $M=N$. Using classical concentration results for example Theorem 4  in \cite{Caponnetto}  one can show that :
$$| \hat{\pazocal{S}}^2_{\mathcal{H},\lambda}(\hat{\nu}_p,\hat{\nu}_q)-\pazocal{S}^2_{\mathcal{H}}(\nu_p,\nu_q)|\leq C \left(\frac{1}{N}+ \mathcal{A}(\lambda)\right) ,$$
with $\lim_{\lambda \to 0} \mathcal{A}(\lambda)=0$, 
and hence 
$|\pazocal{S}^2(\nu_p,\nu_q)-\hat{\pazocal{S}}^2_{\mathcal{H},\lambda}(\hat{\nu}_p,\hat{\nu}_q)|\leq \underbrace{C \left(\frac{1}{N}+ \mathcal{A}(\lambda)\right)}_{\text{Statistical Error}} +  \underbrace{\inf_{u\in \mathcal{H}} \int_{\pazocal{X}}\nor{\nabla_x u(x)-\nabla_x u_{p,q}(x)}^2q(x)dx}_{\text{Approximation error}}$.  The error is therefore dominated by the approximation error and the expressive power of the finite dimensional RKHS. 

\paragraph{Acknowledgments.} The author would like to think Gabriel Peyre and Marco Cuturi  for pointers in the literature as well as Filippo Santambrogio for suggesting using the unconstrained form of $\nor{.}_{\dot{H}^{-1}_{(q)}}$.The author would like to thank also Arthur Gretton and  Bharath  Sriperumbudur  for numerous suggestions and for pointing  issues with using infinite dimensional RKHS.

\bibliographystyle{unsrt}
\bibliography{refs,simplex}

\appendix
\onecolumn
%
\section{Proofs: Sobolev Discrepancy}
\begin{proof}[Proof of Proposition \ref{prop:PrimDual} ]
\noindent 1) \noindent Consider the solution $u_{p,q}$ of the following PDE:
\begin{equation*}
 p(x)-q(x)=- \text{div} (q(x)\nabla_x u(x)) ~~ u\big|_{\partial \pazocal{X}}=0 ,
 \end{equation*}
 physically this PDE means that we are moving the mass $q$ to $p$ following the flows of a velocity field given by $\nabla_x u$.
For any differentiable  function $f$ that vanishes on the boundary of $\pazocal{X}$, with $\nor{f}_{\dot{H}(\nu_q)} \leq 1$, we have by integrating  \eqref{eq:PDEsolution}:
\begin{eqnarray*}
\mathbb{E}_{x\sim \nu_p} f(x) - \mathbb{E}_{x\sim\nu_{q}}f(x)&=&
 \mathbb{E}_{x\sim \nu_q} \scalT{\nabla_x f(x)}{\nabla_x u_{p,q}(x)} \\
&\leq& \nor{\nabla_x h }_{\mathcal{L}^{\otimes d}_2(\pazocal{X},\nu_q)}  \nor{\nabla_x u_{p,q} }_{\mathcal{L}^{\otimes d}_2(\pazocal{X},\nu_q)}\\
\end{eqnarray*}
Let $f^*_{\nu_p,\nu_q}= \frac{u_{p,q}}{\nor{\nabla_x u_{p,q}}_{\mathcal{L}^{\otimes d}_2(\pazocal{X},\nu_q)}}$, we have:\\
$$ \mathbb{E}_{x\sim \nu_p} f^*_{\nu_p,\nu_q}(x) - \mathbb{E}_{x\sim\nu_{q}}f^*_{\nu_p,\nu_q}(x) = \nor{\nabla_x u_{p,q}}_{\mathcal{L}^{\otimes d}_2(\pazocal{X},\nu_q)}=\sqrt{\int_{\pazocal{X}}\nor{\nabla_xu_{p,q}(x)}^2q(x)dx },$$
It follows that for any feasible $f$ of Problem \eqref{eq:Sobolev} we have:
$$\mathbb{E}_{x\sim \nu_p} f(x) - \mathbb{E}_{x\sim\nu_{q}}f(x) \leq   \nor{\nabla_x u_{p,q} }_{\mathcal{L}^{\otimes d}_2(\pazocal{X},\nu_q)}=  \mathbb{E}_{x\sim \nu_p} f^*_{\nu_p,\nu_q}(x) - \mathbb{E}_{x\sim\nu_{q}}f^*_{\nu_p,\nu_q}(x)= \pazocal{S}(\nu_p,\nu_q) $$
It follows therefore that the optimal witness function of \eqref{eq:Sobolev} is  $f^*_{\nu_p,\nu_q}=\frac{u_{p,q}}{\nor{\nabla_x u{p,q}}_{\mathcal{L}^{\otimes d}_2(\pazocal{X},\nu_q)}}$, where $u_{p,q}$ is solution of the advection PDE given in  \eqref{eq:PDEsolution}. The optimal value value of    \eqref{eq:Sobolev} is:
$$S(\nu_p,\nu_q)=\sqrt{\int_{\pazocal{X}}\nor{\nabla_xu_{p,q}(x)}^2q(x)dx }.$$

\noindent 2)  Let  $\bm{v}:\pazocal{X}\to \mathbb{R}^d$ we claim:
\begin{eqnarray}
 \pazocal{S}(\nu_p,\nu_q)=\inf_{\bm{v}}\Big \{ \sqrt{\int_{\pazocal{X}}\nor{\bm{v}(x)}^2 q(x) dx} \text{ subject to}: p(x)-q(x)=- \text{div} (q(x) \bm{v}(x))   \Big \} (P)
 \label{eq:PrimalSobolev}
\end{eqnarray}
\noindent Set the Lagrangian, and assume $f|_{\partial \pazocal{X}}=0$
\begin{eqnarray*}
\pazocal{L}(\bm{v},f)&=&  \sqrt{\int_{\pazocal{X}} \nor{\bm{v}(x)}^2 q(x) dx }+ \int_{\pazocal{X}}f(x)(p(x)-q(x)+div(q \bm{v}(x))) dx\\
&=& \int_{\pazocal{X}} f(x)(p(x)-q(x))dx +  \sqrt{\int_{\pazocal{X}}\nor{\bm{v}(x)}^2 q(x) dx} - \int_{\pazocal{X}} \scalT{\bm{v}(x)}{\nabla_x f(x)}q(x)dx
\end{eqnarray*}
By the convexity of the problem we have that the primal formulation (P) is equal to:
$$\sup_{u}\inf_{\bm{v}} \pazocal{L}(\bm{v},f)= \sup_{u}\Big \{\int_{\pazocal{X}} f(x)(p(x)-q(x))dx +\inf_{\bm{v}}   \sqrt{\int_{\pazocal{X}} \nor{\bm{v}(x)}^2 q(x) dx} - \int_{\pazocal{X}} \scalT{\bm{v}(x)}{\nabla_x f(x)}q(x)dx\Big \}$$
Note that we have two cases:
$$\inf_{\bm{v}}  \sqrt{\int_{\pazocal{X}} \nor{\bm{v}(x)}^2 q(x) dx} - \int_{\pazocal{X}} \scalT{\bm{v}(x)}{\nabla_x f(x)}q(x)dx= -m \text{ if } \sqrt{\int_{\pazocal{X}} \nor{\nabla_x f(x)}^2q(x)dx }> 1 $$
and 
$$\inf_{\bm{v}} \sqrt{ \int_{\pazocal{X}} \nor{\bm{v}(x)}^2 q(x) dx} - \int_{\pazocal{X}} \scalT{\bm{v}(x)}{\nabla_x f(x)}q(x)dx= 0 \text{ if } \sqrt{\int_{\pazocal{X}} \nor{\nabla_x f(x)}^2q(x)dx }\leq 1.$$

Hence we have the primal equal to :
\begin{eqnarray*}
(P)&=& \sup_{f, \sqrt{\int_{\pazocal{X}} \nor{\nabla_x f(x)}^2q(x)dx } \leq 1,~ f|_{\partial \pazocal{X}}=0}\int_{\pazocal{X}} f(x)(p(x)-q(x))dx \\
&=&S(\nu_p,\nu_q)\\
&=& (D).
\end{eqnarray*}
Let $u$ the solution of \eqref{eq:PDEsolution} it follows that the solution $\bm{v}^*=\nabla_x u_{p,q}$, first it is feasible  and $ \sqrt{\int_{\pazocal{X}} \nor{\bm{v}^*(x)}^2 q(x) dx} =\nor{\nabla_x u_{p,q}}_{\mathcal{L}^{\otimes d}_2(\pazocal{X},\nu_q)}= \pazocal{S}(\nu_p,\nu_q)$.\\
\end{proof}

\begin{proof}[Proof of Proposition \ref{prop:BoundW2}]
For the upper bound $W_2(\nu_q,\nu_p)\leq 2  \pazocal{S}(\nu_p,\nu_q)$ the proof is given in Lemma 1 of \cite{peyre2016comparison} : $$\mathcal{W}_2(\nu_q,\nu_p)\leq 2  \nor{\nu_p-\nu_q}_{\dot{H}^{-1}(\nu_q)}=2\pazocal{S}(\nu_p,\nu_q).$$ 
For the lower  bound we adapt the proof of  \cite{santambrogio2015optimal} given for $\nor{.}_{\dot{H}^{-1}(\pazocal{X})}$. Let $f^*_{\nu_p,\nu_q}$ be the optimal Sobolev witness function we have $\nor{\nabla_x f^*_{\nu_p,\nu_q}}_{\mathcal{L}_2(\pazocal{X},\nu_q)}=1$, and $f^*_{\nu_p,\nu_q}|_{\partial \pazocal{X}}=0$. Let $(\mu_{q_t},\bm{v}_t)$ the solutions  of the dynamic formulation of Benamou and Bernier. 
Note that $\mu_{q_t}$ are absolutely continuous and their densities remain bounded from above and below by the same constants (see Prop 7.29 and Prop 7.30  Santambrogio book).

\begin{eqnarray*}
\pazocal{S}(\nu_p,\nu_q)&=&\int_{\pazocal{X}} f^*_{\nu_p,\nu_q}d(\nu_p-\nu_q)= \int_{0}^1 \frac{d}{dt}\left(\int_{\pazocal{X}} f^*_{\nu_p,\nu_q}(x)d\mu_{q_t}(x)\right) dt \\
&=& \int_{0}^1\int_{\pazocal{X}} \scalT{\nabla_x f^*_{\nu_p,\nu_q}(x)}{\bm{v}_t(x)}q_t(x) dxdt\\
&\leq& \left(\int_{0}^1\int_{\pazocal{X}}  \nor{\nabla_x f^*_{\nu_p,\nu_q}(x)}^2 q_t(x) dx dt \right)^{\frac{1}{2}}\left(\int_{0}^1\int_{\pazocal{X}}  \nor{ \bm{v}_t(x)}^2 q_t(x) dx dt \right)^{\frac{1}{2}}\\
&=& \left(\int_{0}^1\int_{\pazocal{X}}  \nor{\nabla_x f^*_{\nu_p,\nu_q}(x)}^2 q(x) \frac{q_t(x)}{q(x)} dx dt \right)^{\frac{1}{2}}{W}_2(\nu_q,\nu_p)\\
&\leq&  \sqrt{\frac{b}{a}} \left( \int_{\pazocal{X}}  \nor{\nabla_x f^*_{\nu_p,\nu_q}(x)}^2 q(x) dx \right)^{\frac{1}{2}}{W}_2(\nu_q,\nu_p)\\
&=& \sqrt{\frac{b}{a}}{W}_2(\nu_q,\nu_p).
\end{eqnarray*}
\end{proof}

\begin{proof} [Proof of Lemma \ref{lem:Ssquared}]Let $u_{p,q}$ be the solution of the following PDE:
\begin{equation}
p(x)-q(x)= -div(q(x)\nabla_x u_{p,q}(x)),
\label{eq:advec}
\end{equation}
with boundary condition $\scalT{\nabla_x u_{p,q}(x)}{n(x)}=0$ on $\partial \pazocal{X}$. We know that :
$$\pazocal{S}^2(\nu_p,\nu_q)= \int_{\pazocal{X}} \nor{\nabla_x u_{p,q}(x)}^2 q(x)dx.$$

\textbf{Step 1.} Setting $u=u_{p,q}$. Let us first show that $$L(u_{p,q})=2\int_{\pazocal{X}} u_{p,q}(x) (p(x)-q(x))dx-\int_{\pazocal{X}}\nor{\nabla_x u_{p,q}(x)}^2 q(x)dx= \int_{\pazocal{X}} \nor{\nabla_x u_{p,q}(x)}^2 q(x)dx.$$
\begin{align*}
L(u_{p,q})&= -2\int_{\pazocal{X}}u_{p,q}(x)div(q(x)\nabla_x u_{p,q}(x))-\int_{\pazocal{X}}\nor{\nabla_x u_{p,q}(x)}^2 q(x)dx \text{ (Using  Equation } \eqref{eq:advec}) \\
&=-2\left(\int_{\pazocal{X}}-\nor{\nabla_x u_{p,q}(x)}^2q(x)dx \right)-\int_{\pazocal{X}}\nor{\nabla_x u_{p,q}(x)}^2 q(x)dx \text{ ( Divergence Theorem )}\\
&= \int_{\pazocal{X}}\nor{\nabla_x u_{p,q}(x)}^2 q(x)dx\\
&=\pazocal{S}^2(\nu_p,\nu_q).
\end{align*}

\textbf{Step 2.} Let us show that $u_{p,q}$ is an optimizer of this loss. Meaning that for all $u$ we have $L(u)\leq L(u_{p,q})$.
\begin{align*}
L(u)&=2\int_{\pazocal{X}} u(x) (p(x)-q(x))dx-\int_{\pazocal{X}}\nor{\nabla_x u(x)}^2 q(x)dx\\
&= -2\int_{\pazocal{X}} u(x)div(q(x)\nabla_x u_{p,q} )-\int_{\pazocal{X}}\nor{\nabla_x u(x)}^2 q(x)dx  \text{ (Using  Equation } \eqref{eq:advec}) \\
&= 2\int_{\pazocal{X}}\scalT{\nabla_x u(x)}{u_{p,q}(x)}q(x)dx-\int_{\pazocal{X}}\nor{\nabla_x u(x)}^2 q(x)dx \text{ ( Divergence Theorem )}\\
&=  2\int_{\pazocal{X}}\scalT{\nabla_x u(x)}{u_{p,q}(x)}q(x)dx-\int_{\pazocal{X}}\nor{\nabla_x u(x)}^2 q(x)dx-L(u_{p,q})+L(u_{p,q})\\ 
&= L(u_{p,q})+ \int_{\pazocal{X}}\left(2\scalT{\nabla_x u(x)}{u_{p,q}(x)}-\nor{\nabla_x u(x)}^2 -\nor{\nabla_x u_{p,q}(x)}^2\right)q(x)dx\\
&= L(u_{p,q}) -\int_{\pazocal{X}}\nor{\nabla_x u(x)-\nabla_x u_{p,q}(x)}^2q(x)dx\\
&\leq L(u_{p,q})=\pazocal{S}^2(\nu_p,\nu_q).
\end{align*}
with equality when $u=u_{p,q}$.
\end{proof}

\section{Proofs: Kernel Sobolev Discrepancy}
\begin{proof}[Proof of Proposition \ref{prop:SolKernel}]

It is easy to see that for $f\in \mathcal{H}$,
$$\mathbb{E}_{x\sim\nu_p}f(x)-\mathbb{E}_{x\sim \nu_q}f(x)=\scalT{\boldsymbol{f}}{\mu(\nu_p)-\mu(\nu_q)},$$
where $\mu(\nu_p)$ and $\mu(\nu_q)$ are the KME of $\nu_p$ and $\nu_{q}$.
On the other hand:
\begin{align*}
\mathbb{E}_{x\sim \nu_q} \nor{\nabla_xf(x)}^2=\scalT{\boldsymbol{f}}{D(\nu_q)\boldsymbol{f}}
\end{align*}
where $D(\nu_q)$ is the KDGE of $\nu_q$ (as defined in Equation \eqref{eq:KDGE}).
Hence we have under the assumption that $D(\nu_q)$ is invertible:
\begin{align*} 
\pazocal{S}_{\mathcal{H}}(\nu_p,\nu_q)&=\sup_{f\in \mathcal{H}, \scalT{f}{D(\nu_q)f}_{\mathcal{H}}\leq 1}\scalT{f}{\mu(\nu_p)-\mu(\nu_q)}\\
&=\sup_{\boldsymbol{g} \nor{\boldsymbol{g}}\leq 1} \scalT{\boldsymbol{g}}{D^{-\frac{1}{2}}(\nu_q)\left(\mu(\nu_p)-\mu(\nu_q)\right)}\\
&=\nor{D^{-\frac{1}{2}}(\nu_q)\left(\mu(\nu_p)-\mu(\nu_q)\right)},
\end{align*}
and 
$$\boldsymbol{f}^{\mathcal{H}}_{\nu_p,\nu_q}=\frac{1}{\pazocal{S}_{\mathcal{H}}(\nu_p,\nu_q)  }\left[D(\nu_{q})\right]^{-1} \left(\mu(\nu_p)-\mu(\nu_q)\right),$$
and $\boldsymbol{u}^{\mathcal{H}}_{p,q}=\left[D(\nu_{q})\right]^{-1} \left(\mu(\nu_p)-\mu(\nu_q)\right)$.

\end{proof}

\begin{proof} [Proof of Proposition \ref{prop:dense}] We have: $\pazocal{S}_{\mathcal{H}}(\nu_p,\nu_q)\leq \pazocal{S}(\nu_p,\nu_q)$
\end{proof}

\section{Proofs:Regularized Kernel Sobolev Discrepancy}
%
%

\begin{proof}[Proof of Theorem \ref{pro:RegSobWit}] Let $\boldsymbol{u}^{\lambda}_{p,q} \in \mathcal{H}$ be the solution of :
$$(D(\nu_q)+\lambda I )\boldsymbol{u}^{\lambda}_{p,q}=\mu(\nu_{p})-\mu(\nu_q)$$
We know that the solution of \eqref{eq:ipmSobO} satisfies (for a proof it is similar to \ref{prop:SolKernel} just adding the regularization $\lambda>0$):
\begin{eqnarray*}
\pazocal{S}^2_{\mathcal{H},\lambda}(\nu_p,\nu_q)&=& \scalT{\mu(\nu_{p})-\mu(\nu_q)}{(D(\nu_q)+\lambda I)^{-1}(\mu(\nu_{p})-\mu(\nu_q))}_{\mathcal{H}}\\
&=&\scalT{\boldsymbol{u}^{\lambda}_{p,q}}{D(\nu_q)\boldsymbol{u}^{\lambda}_{p,q}} +\lambda \nor{\boldsymbol{u}^{\lambda}_{p,q}}^2\\
&=& \int_{\pazocal{X}}\nor{\nabla_x {u}^{\lambda}_{p,q}(x)}^2q(x)dx +\lambda \nor{\boldsymbol{u}^{\lambda}_{p,q}}^2.
\end{eqnarray*}
Let $$L(u,\lambda)=2\scalT{\boldsymbol{u}}{\mu(\nu_p)-\mu(\nu_q)}-\scalT{\boldsymbol{u}}{(D(\nu_q)+\lambda I) \boldsymbol{u}}$$
\textbf{Step 1.} 

\begin{align*}
L({u}^{\lambda}_{p,q},\lambda)&=2\scalT{\boldsymbol{u}^{\lambda}_{p,q}}{\mu(\nu_p)-\mu(\nu_q)}_{}-\scalT{\boldsymbol{u}^{\lambda}_{p,q}}{(D(\nu_q)+\lambda I) \boldsymbol{u}^{\lambda}_{p,q}} \\
&=2\scalT{\boldsymbol{u}^{\lambda}_{p,q}}{(D(\nu_q)+\lambda I)\boldsymbol{u}^{\lambda}_{p,q}}-\scalT{\boldsymbol{u}^{\lambda}_{p,q}}{(D(\nu_q)+\lambda I) \boldsymbol{u}^{\lambda}_{p,q}}\\
&= \scalT{\boldsymbol{u}^{\lambda}_{p,q}}{(D(\nu_q)+\lambda I) \boldsymbol{u}^{\lambda}_{p,q}}\\
&= \pazocal{S}^2_{\mathcal{H},\lambda}(\nu_p,\nu_q).
\end{align*}

\textbf{Step 2.} Let us show that $\boldsymbol{u}^{\lambda}_{p,q}$ is an optimizer of this loss. Meaning that for all $u$ we have $L(u,\lambda)\leq L({u}^{\lambda}_{p,q},\lambda)$.
\begin{align*}
L(u,\lambda)- L({u}^{\lambda}_{p,q},\lambda)&=2\scalT{\boldsymbol{u}}{\mu(\nu_p)-\mu(\nu_q)}-\scalT{\boldsymbol{u}}{(D(\nu_q)+\lambda I) \boldsymbol{u}}-\scalT{\boldsymbol{u}^{\lambda}_{p,q}}{(D(\nu_q)+\lambda I) \boldsymbol{u}^{\lambda}_{p,q}}\\
&=2\scalT{\boldsymbol{u}}{(D(\nu_q)+\lambda I )\boldsymbol{u}^{\lambda}_{p,q}} -\scalT{\boldsymbol{u}}{(D(\nu_q)+\lambda I) \boldsymbol{u}}-\scalT{\boldsymbol{u}^{\lambda}_{p,q}}{(D(\nu_q)+\lambda I) \boldsymbol{u}^{\lambda}_{p,q}}\\
&=- \scalT{\boldsymbol{u}- \boldsymbol{u}^{\lambda}_{p,q}}{(D(\nu_q)+\lambda I )(\boldsymbol{u}-\boldsymbol{u}^{\lambda}_{p,q})}\\
&=-\scalT{\boldsymbol{u}- \boldsymbol{u}^{\lambda}_{p,q}}{D(\nu_q)(\boldsymbol{u}-\boldsymbol{u}^{\lambda}_{p,q})}-\lambda \nor{\boldsymbol{u}-\boldsymbol{u}^{\lambda}_{p,q}}^2\\
&= -\int_{\pazocal{X}}\nor{\nabla_x u(x)-\nabla_x {u}^{\lambda}_{p,q}(x)}^2q(x)dx-\lambda \nor{\boldsymbol{u}-\boldsymbol{u}^{\lambda}_{p,q}}^2\\
&=-\nor{(D(\nu_q))^{\frac{1}{2}}(\boldsymbol{u}-\boldsymbol{u}^{\lambda}_{p,q}) }^2-\lambda \nor{\boldsymbol{u}-\boldsymbol{u}^{\lambda}_{p,q}}^2\\
&\leq 0.
\end{align*}
with equality when $u=\boldsymbol{u}^{\lambda}_{p,q}$.
1) to 5) follow immediately from the proof above.
\end{proof}

\section{Convergence}

\begin{proof}[Proof of Lemma \ref{lem:compa}]
1) We have from Lemma \ref{lem:Ssquared} (Equation \eqref{eq:comparaisonequality}):
\begin{align*} 0\leq \pazocal{S}^2(\nu_p,\nu_q)- \pazocal{S}^2_{\mathcal{H}}(\nu_p,\nu_q)&=\int_{\pazocal{X}}\nor{\nabla_x u^{\mathcal{H}}_{p,q}(x)-\nabla_x u_{p,q}(x)}^2q(x)dx\\
&\leq  \inf_{u\in \mathcal{H}} \int_{\pazocal{X}}\nor{\nabla_x u(x)-\nabla_x u_{p,q}(x)}^2q(x)dx
\end{align*}

2) 
Let $$L(u,\lambda)= 2\scalT{\boldsymbol{u}}{\mu(\nu_p)-\mu(\nu_q)}-\scalT{\boldsymbol{u}}{(D(\nu_q)+\lambda I) \boldsymbol{u}}$$ 
and 
$$\hat{L}(u,\lambda)= 2\scalT{\boldsymbol{u}}{\hat{\mu}(\nu_p)-\hat{\mu}(\nu_q)}-\scalT{\boldsymbol{u}}{(\hat{D}(\nu_q)+\lambda I) \boldsymbol{u}}.$$
Note $\delta=\mu(\nu_p)-\mu(\nu_q)$ and $\hat{\delta}=\hat{\mu}(\nu_p)-\hat{\mu}(\nu_q)$.\\
From Theorem \ref{pro:RegSobWit} point 4) we have: For any $\lambda>0$, and any $u \in \mathcal{H}$:
  $$L({u}^{\lambda}_{p,q},\lambda)-L(u,\lambda) = \nor{(D(\nu_q))^{\frac{1}{2}}(\boldsymbol{u}-\boldsymbol{u}^{\lambda}_{p,q}) }^2+\lambda \nor{\boldsymbol{u}-\boldsymbol{u}^{\lambda}_{p,q}}^2$$
In particular for the unregularized case $\lambda_0=0$ we have $L(u^{\lambda_0}_{p,q},\lambda_0)= L(u^{\mathcal{H}}_{p,q},0)=\pazocal{S}^2_{\mathcal{H}}(\nu_p,\nu_q)$.
Hence for this particular case we have for any $u \in \mathcal{H}$

$$\pazocal{S}^2_{\mathcal{H}}(\nu_p,\nu_q)-L(u,0) = L(u^{\mathcal{H}}_{p,q},0)-L(u,0)=\nor{(D(\nu_q))^{\frac{1}{2}}(\boldsymbol{u}-\boldsymbol{u}^{\lambda}_{p,q}) }^2 $$
   \begin{align*}
  \hat{\pazocal{S}}^2_{\mathcal{H},\lambda}(\hat{\nu}_p,\hat{\nu}_q)-\pazocal{S}^2_{\mathcal{H}}(\nu_p,\nu_q)&= \hat{L}(\hat{u}^{\lambda}_{p,q},\lambda)-L(\hat{u}^{\lambda}_{p,q},\lambda)+L(\hat{u}^{\lambda}_{p,q},\lambda)-L({u}^{\mathcal{H}}_{p,q},0)\\
  &=2\scalT{\boldsymbol{\hat{u}}^{\lambda}_{p,q}}{\hat{\delta}-\delta}-\scalT{\boldsymbol{\hat{u}}^{\lambda}_{p,q}}{(\hat{D}(\nu_q)-D(\nu_q))\boldsymbol{\hat{u}}^{\lambda}_{p,q}}\\&
  + L(\hat{u}^{\lambda}_{p,q},0)+\lambda \nor{\boldsymbol{\hat{u}}^{\lambda}_{p,q}}^2-L({u}^{\mathcal{H}}_{p,q},0)\\
  &=2\scalT{\boldsymbol{\hat{u}}^{\lambda}_{p,q}}{\hat{\delta}-\delta}-\scalT{\boldsymbol{\hat{u}}^{\lambda}_{p,q}}{(\hat{D}(\nu_q)-D(\nu_q))\boldsymbol{\hat{u}}^{\lambda}_{p,q}}+\lambda \nor{\hat{u}^{\lambda}_{p,q}}^2 \\
  &- \nor{\sqrt{D(\nu(q))}(\boldsymbol{\hat{u}}^{\lambda}_{p,q}-\boldsymbol{u}^{\mathcal{H}}_{p,q})}^2
\end{align*}
\begin{align*}
| \hat{\pazocal{S}}^2_{\mathcal{H},\lambda}(\hat{\nu}_p,\hat{\nu}_q)-\pazocal{S}^2_{\mathcal{H}}(\nu_p,\nu_q)| &\leq \nor{\delta-\hat{\delta}}\nor{\boldsymbol{\hat{u}}^{\lambda}_{p,q}}+(1+\lambda )\nor{\boldsymbol{\hat{u}}^{\lambda}_{p,q}}^2\nor{D(\nu_q)-\hat{D}(\nu_q)}_{op}\\
&+  \nor{\sqrt{D(\nu(q))}(\boldsymbol{\hat{u}}^{\lambda}_{p,q}-\boldsymbol{u}^{\mathcal{H}}_{p,q})}^2
\end{align*}
\end{proof}

\end{document}